\documentclass{dynafront2025} 

\newcommand{\bigdotcup}{\bigcup\mkern-13.5mu\cdot\mkern6mu}
\newcommand{\A}{\mathcal{A}}

\newcommand{\Z}{\mathcal{Z}}

\newcommand{\G}{\mathcal{G}}
\newcommand{\State}{\mathcal{S}}

\newcommand{\Hist}{\mathcal{H}}

\newcommand{\Ell}{\mathcal{L}}
\newcommand{\I}{\mathcal{I}}

\newcommand{\E}{\mathbb{E}}

\title[Generative Sampling in Incomplete Information Games]{IFlowNets: Extending Generative Samplers to Learn Strategies in Incomplete Information Games}

\optauthor{%
\Name{Conor M. Artman} \Email{artman1@llnl.gov}\\
\addr AI Research Group, Lawrence Livermore National Laboratory
\AND
\Name{Nicholas Di} \Email{nd56@rice.edu}\\
\addr Deptartment of Statistics, Rice University,\\%
\addr AI Research Group, Lawrence Livermore National Laboratory%
\AND
\Name{Scott Perkins} \Email{perkins35@llnl.gov}\\
\addr Strategic Competition Analysis Group, Lawrence Livermore National Laboratory%
}

\begin{document}

\maketitle

\begin{abstract}%
  
  While many algorithms blend reinforcement learning (RL) with 
  counterfactual regret (CFR) methods to leverage tradeoffs in computational speed and performance, there are fewer investigations into 
  generative sampling frameworks in game theoretic applications in incomplete information games. We extend a generative flow network framework, \textit{Adversarial Flow Networks} (AFlowNets), to incomplete information games, called \textit{Information Flow Networks} (IFNs). 
  We prove that previously established constraints for generative flow networks in complete information games are inadmissible for obtaining valid densities (corresponding to player strategies) and a valid training objective. 
  We show that our proposed generalization, IFlowNets, alleviates this issue and strictly generalizes AFlowNets. 
  In preliminary results for three standard game environments, IFlowNets perform comparably to or better than Outcome Sampling Monte Carlo Counterfactual Regret (OS-MCCFR) and standard RL-based methods in performance and speed. 





\end{abstract}


\section{Introduction}

Generative flow networks (GFlowNets) operate by converting \textit{flow matching constraints} over directed acyclic graphs (DAGs) into an objective function, which is 
then used to fit a sampler that respects terminal, reward-proportional sampling via direct, gradient-based search. In DAGs with 
intragraph uncertainty, \citet{jiralerspong_expected_2024} extended GFlowNets to handle uncertain intragraph transitions by using 
alternating samplers that use expectations to smooth over uncertainty,  called \textit{Expected Flow Networks} (EFlowNets). 
As an application, \citet{jiralerspong_expected_2024} showed that one can alternate these samplers as agents playing a complete information extensive form game, 
called adversarial flow networks (AFlowNets). AFlowNet agents learn a policy that samples proportional to the \textit{expected utility} under environment and player transition dynamics,
which produced results as good or better than AlphaZero augmented by Monte Carlo Tree Search (as well as standard RL 
methods, e.g., DQN, A2C and PPO). Importantly, AFlowNets rely solely on generating trajectories and observing outcomes 
strictly during training. In contrast, many RL-centric methods, such as AlphaZero, \textit{need} approximate recursive search to 
perform well. Further, the approach appears promising, because in perfect information environments, AFlowNets learn to play effectively in the vast majority of states \textit{without} any 
form of inference-time search and human heuristics, despite only seeing a small fraction of these states. Despite these observations, generative flow network-based methods have not been investigated in the context of non-stationary, sequential, or incomplete 
information games. In this new context, we theoretically demonstrate that our extension, Information Flow Networks (IFlowNets), generalizes AFlowNets
and validly reproduces the expected flow-matching property. The expected flow matching property is known to induce a generalized Nash equilibrium called a \textit{quantal response equilibrium} (QRE), 
and are typically difficult to estimate \citep{jiralerspong_expected_2024, bland_quantal_2025}. (We defer studying of QRE estimated by IFlowNets
compared to classic methods for future work, but see Appendix \ref{appdx:qre} for a brief description of IFlowNets as QRE.) Our contributions are as follows.



\begin{enumerate}
    \item We prove that directly applying \citet{jiralerspong_expected_2024}'s approach to handling intragraph uncertainty is impossible: doing so \textit{invalidates} properties necessary for expected reward-proportional sampling in incomplete information settings. 
    \item We generalize AFlowNets to incomplete information settings and demonstrate they preserve desired properties in the literature, e.g., the flow matching property.
    \item We test our IFlowNet formulation in three standard incomplete information environments and show preliminary results against Outcome Sampling Monte Carlo CFR, deep CFR, and neural fictitious self play (NFSP).
\end{enumerate}

\section{Background \& Notation}
$[N]:=\{1, \ldots, N\}$ 
is the set of players, indexed by $i$. $\Hist$ represents sequences denoting the possible \textit{histories or trajectories} of 
actions. For $h, h' \in \Hist$, let $h'$ be a longer history than $h$. 
Starting at 
any $h$ and choosing an action $a$ leads to the next history, $h' := (h,a)$. $h$ is a ``prefix" of $h'$. $\Z \subseteq \Hist$ are terminal histories, and each $z$ is a sequence of actions representing a complete path from the start of the game to a terminal node. 
The set of all actions for the current player 
acting at a non-terminal history $h$ is $\A(h) := \{a : (h,a) \in \Hist\}$. Actions available player $i$ 
at $h$ is denoted $\A_{i}(h)$, $h \in \Hist \setminus \Z$. In RL terminology, if all histories were 
fully observable at every time-step, we could just call them states, denoted $s \in \State$. However, due to uncertainty in the game, agents may 
not know what state they are in at a particular decision point. An information state $I$ (\textit{infostate} or \textit{infoset}), is an 
aggregate state over all the possible states a player could be in at a decision point. Precisely, $\forall h, h' \in I$ 
with $h\neq h'$, $h$ and $h'$ are indistinguishable to an agent. Equivalently, $\A(h) = \A(h')$. We distinguish between 
different agents' infostates by \textit{information partitions}, $\I_i$. $\I_i$ partitions 
the set of all histories into player $i$'s potential histories, also with the property $\A(h) = \A(h')$ whenever $h, h' \in \I_i$. 
Subscripts denote which agent's action space and infostates we are referring to. 
Each player has a utility function $u_i: \Z \mapsto \mathbb{R}$, and we interpret reward functions $R_i$ to be transformations of utility functions. Utility functions 
determine classes of games, e.g., if $[N] =\{1,2\}$ and $u_1 = - u_2$, the game is a two-player zero-sum game. We will 
be focused on zero-sum extensive form games of incomplete information, which are DAG-structured games. We assume 
perfect recall throughout. The set of child nodes of an (info)states $\text{Ch}(\cdot)$ is all 
next-possible transitions from that history or infostate. 
Reward proportional sampling is 
$P_i(z) \propto R_i(z)$, for sampler (or strategy) $P_i$;
 expected reward proportional sampling 
refers to $P_i(z) \propto \E_{-i}[R(z)]$, where $-i$ is all players \textit{except} $i$. $F$ refers to 
the ``flow'' function, which maps (info)states to unnormalized probability mass.

\subsection{Review of Expected \& Adversarial Flow Networks}

Due to space limits, we defer review of expected detailed balance (EDB) and AFlowNets to Appendix \ref{appdx:review}. 

\section{Methodology}

Before discussing technical details, we summarize our findings as follows.

\begin{enumerate}
    \item The EDB constraints for AFlowNets are inadmissible for incomplete information games, i.e., the same expectation-based smoothing approach as \citet{jiralerspong_expected_2024} cannot work.
    \item Consequently, we cannot rely on the theorems developed by \citet{jiralerspong_expected_2024}. 
    Therefore, we must fix and extend their constraints to justify using the usual trajectory balance objective for generative flow network models \citep{madan_learning_nodate}. 
\end{enumerate}



\subsection{Aggregating Flows over Infostates.}

Below we discuss the two main steps for fixed and generalizing AFlowNets to incomplete information games. (See Appendix \ref{appdx:new-cases}, Figure \ref{fig:case-diagrams} for graphics visualizing the two new cases 
that must be accounted for.) 

\paragraph{Generalization \#1: Information Set Aggregation.} Our first insight is that the EDB are missing an extra necessary constraint to be applied to incomplete information games that we refer to as \textit{infostate aggregation property}: $F(I) = \sum_{h \in I} F(h)$. While this appears similar to flow matching, this is actually a self-consistency condition, whereas flow matching asserts that all inflows must match all outflows, $F(I) = \sum_{I' \in Ch(I)} F(I')$. 
(see Appendix \ref{appdx:infostate-agg} for more discussion.)


\paragraph{Generalization \#2: Higher Intragraph Uncertainty in Incomplete Information games.} 
Due to agent-level uncertainty about which history an agent occupies in an infostate, there is another layer of uncertainty to aggregate or smooth over, which we refer to as \textit{double intragraph uncertainty}. In incomplete information games, we find that applying the same principles \citet{jiralerspong_expected_2024} used to derive the EDB do not work in the incomplete information setting
(see Theorem \ref{thm:dedb-fails}). 
(See Figure \ref{fig:case-diagrams} for diagrams describing new uncertainties introduced by incomplete information games.)

To illustrate why these generalizations are necessary, we naively extend \citet{jiralerspong_expected_2024}'s to incomplete information settings and re-apply the idea of using expectations to smooth over intragraph uncertainty. We refer to these as the \textit{Double Expected Detailed Balance} (DEDB) constraints.

\begin{align}
  F(z) &= R(z)\label{eqn:D1}\tag{D1}, \hspace{5mm} (\forall z \in \mathcal{Z}) \\ 
  F(I) &= \E_{I' \sim P_{env}(\cdot \vert I)} \Big [ F(I') \Big]\label{eqn:D2}\tag{D2}, \hspace{5mm} ( \forall h, h' \in \mathcal{I}_{env}) \\
  F(I) P_{agent}(a \vert I) & = \E_{h' \sim P_{env}} \Big [F(h')\Big ]\label{eqn:D3}\tag{D3}, \hspace{5mm} ( \forall a, h', I \in \mathcal{I}_{agent}) \\
  F(I) &= \sum_{h \in I} F(h)\label{eqn:D4}\tag{D4}, \hspace{5mm} ( \forall h, I \in \mathcal{I})
\end{align}

While we augment \citet{jiralerspong_expected_2024}'s constraints to respect infostates with \ref{eqn:D4}, the key issue remains in constraint \ref{eqn:D3}. Theorem $1$ implies $P_{env}$ ceases to be a 
valid probability mass function in order to satisfy flow matching, so it is impossible to learn valid sampling strategies.

\begin{theorem}(\textit{Expectation-Based Aggregation over Infosets via DEDB Invalidates Flow Matching})\label{thm:dedb-fails}(see Appendix \ref{appdx:exp-failure} for proof of Theorem 1.)
\end{theorem}

In response, we propose the \textit{generalized expected detailed balance (GEDB) constraints}. 
Our GEDB constraints over DAGs of incomplete information games accomplish the following.\\ 1. The GEDB constraints successfully induce expected flow matching property over \textit{infostates}.
\\2. The GEDB automatically recover the EDB constraints in complete information settings. \\3.   (D0-D3 in \citep{jiralerspong_expected_2024}). 
The last item follows directly if items $1$ and $2$ are satisfied, so we focus on the first two \citep{jiralerspong_expected_2024}.


\subsection{Generalized Expected Detailed Balance}
In addition to adding the necessary infostate aggregation constraint, we find that constraint \ref{eqn:C3} corrects \ref{eqn:D3}. This allows for flow matching to hold over infostates, which paves the way for expected reward proportional sampling in incomplete information games. 

\begin{align}
    F(z) &= R(z)\label{eqn:C1}\tag{C1}, \hspace{5mm} (\forall z \in \mathcal{Z}) \\ 
    F(I) &= \E_{I' \sim P_{env}(\cdot \vert I)} \Big [ F(I') \Big]\label{eqn:C2}\tag{C2}, \hspace{5mm} ( \forall h, h' \in \mathcal{I}_{env}) \\
    F(I) P_{agent}(a \vert I) & = \sum_{h' \ni a} F(h')\label{eqn:C3}\tag{C3}, \hspace{5mm} ( \forall a, h', I \in \mathcal{I}_{agent}) \\
    F(I) &= \sum_{h \in I} F(h)\label{eqn:C4}\tag{C4}, \hspace{5mm} ( \forall h, I \in \mathcal{I})
\end{align}


\noindent In the next lemma, the GEDB constraints imply that the graph-level flow function $F$ satisfies expected flow matching over infosets.
As a consequence, it is valid to use the well-known \textit{trajectory balance} objective function to fit agent policies over infostates rather than complete information states \citep{madan_learning_nodate,jiralerspong_expected_2024}. (See Appendix \ref{appdx:implementation} 
for trajectory balance objective and implementation details.)
\begin{lemma}(\textit{Generalized Expected Detailed Balance Conditions Imply Expected Flow Matching over Infosets.})\label{lemma:FM}  See Appendix \ref{appdx:infoset-fm}.
\end{lemma}




The next Lemma indicates that the GEDB allow IFlowNets to automatically switch between complete and incomplete information games by reparameterizing states. 

\begin{lemma}(\textit{The GEDB Constraints Recover the EDB Constraints in Complete Information Settings.})\label{lemma:gedb_generalize_edb} See Appendix \ref{appdx:gedb-edb} for proof.
\end{lemma}  

   




\section{Experiments}

We evaluate our IFlowNet formulation in the following tasks: \textit{Weighted RPS, Kuhn Poker, Leduc Poker}. In the latter two,
we compare to Neural Fictitious Self Play (NFSP), Deep CFR, and Outcome Sampling Monte Carlo CFR (OC-MCCFR). 

\paragraph{Weighted Rock-Paper-Scissors}

In Weighted Rock-Paper-Scissors (RPS+) the wins and losses are doubled if Scissors is played. RPS+ is played in extensive form, 
so Player $1$ chooses a move first and Player $2$ observes Player $1$'s choice after the game ends. Player $2$ 
is always at an incomplete information state, and Player $1$ must learn to play over the uncertainty of Player $2$'s strategy.
The Nash equilibrium
in RPS+ is $(R,P,S)=(.4, .4, .2)$ for both players. \footnote{Note:
in repeated play (such as our setup) RPS+ is deceptively difficult for RL and is a benchmark 
for multi-agent learners' ability to adapt in non-stationary, incomplete information environments \citep{lanctot_population-based_2023}.}
IFlowNets find the Nash strategy (Figure \ref{fig:weighted-rps}) and approach minimal exploitability (see Figure \ref{fig:weighted-rps-kld}).

\begin{figure}[htp] 
    \centering
    \subfigure[]{\includegraphics[width=0.4\textwidth]{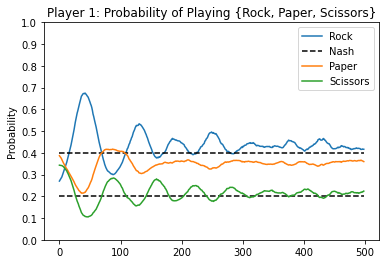}}
    \subfigure[]{\includegraphics[width=0.4\textwidth]{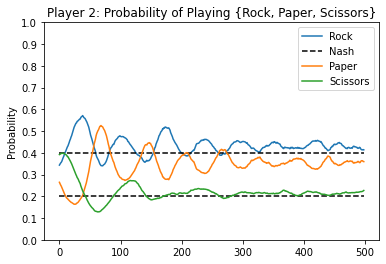}}
    \caption{Policies per training iteration in RPS+. (a) Player $1$ (b) Player $2$}
    \label{fig:weighted-rps}
\end{figure}

\paragraph{Kuhn Poker} Kuhn poker is a simplified form of poker as a simple model zero-sum two-player 
imperfect-information games (\cite{kuhn_contributions_2016}). In Kuhn poker, the deck includes only three playing 
cards: a King, Queen, and Jack. One card is dealt to each player, which may place bets similarly to a standard poker. 
If both players bet or both players pass, then the player with the higher card wins. Otherwise, the player who bet wins.

\begin{table}[h]
    \begin{center}
        \begin{tabular}{ c c c } 
        \hline
        Algorithm & Exploitability & Iters./Sec. \\
        \hline 
        IFlowNet & 0.087 & 340 \\ 
        OS-MCCFR & \textbf{0.068} & \textbf{1211} \\ 
        NFSP & 0.417 & 0.35 \\ 
        DeepCFR & 0.451 & 0.008 \\ 
        
        \hline
    \end{tabular}
    \end{center}
    \caption{Kuhn Poker with $n=2$ players after $10,000$ iterations.}
\end{table}

\paragraph{Leduc Poker} Leduc Hold’em is 2-player Limit Texas Hold’em, with
2 rounds and a six card deck (Jack, Queen, and King in 2 suits). At the beginning of the game, each player 
receives one card and, after betting, a public card is revealed. After another round, the player 
with the best hand wins and receives a reward 1 and the loser receives -1. At any time, players can fold.

\begin{table}[h]
    \begin{center}
        \begin{tabular}{ c c c } 
        \hline
        Algorithm & Exploitability & Iters./Sec. \\
        \hline
        IFlowNet & \textbf{1.287} & \textbf{55} \\ 
        OS-MCCFR & 2.725 & 7.62 \\ 
        NFSP & 2.691 & 0.09 \\ 
        DeepCFR & 1.431 & 0.01 \\ 
        \hline
    \end{tabular}
    \end{center}
    \caption{Leduc Poker with $n=2$ players after $10,000$ iterations.}
\end{table}

\section{Discussion \& Conclusion}

We show IFlowNets generalize AFlowNets to incomplete information games and show preliminary results comparable to 
or better than closely-related MCCFR and RL variants. In Kuhn Poker, IFlowNets perform comparably to OS-MCCFR in exploitability but is slower. In Leduc Poker, IFlowNets perform best and
are faster to compute compared to relevant CFR and Deep RL variants. 

\section*{Acknowledgments}

This work was performed under the auspices of the U.S. Department of Energy by Lawrence Livermore
National Laboratory under Contract DE-AC52-07NA27344 and was supported by the LLNL-LDRD Program 
under Project No. 25-FS-026. LLNL-CONF-2010743. 

\section*{Government Use License Notice}

This manuscript has been authored by Lawrence Livermore National Security, LLC under Contract No. DE-AC52-07NA2
7344 with the US. Department of Energy. The United States Government retains, and the publisher, by accepting the
article for publication, acknowledges that the United States Government retains a non-exclusive, paid-up, irrevocable,
world-wide license to publish or reproduce the published form of this manuscript, or allow others to do so, for United
States Government purposes. 
\bibliography{references}

\newpage
\appendix

\section{Extended Note on Necessity of Generalization \#1: Information Set Aggregation}\label{appdx:infostate-agg}

In incomplete information games, information set aggregation (Eqn. \ref{eqn:D4}) is not a trivial condition to omit. 
This reasoning behind this is the following. \textit{All} generative flow network methodologies can be likened to direct-search variants of Hamiltonian MCMC \citep{deleu_generative_2023}. 
Consequently, GFlowNet models target a marginal distribution of the invariant distribution over a DAG: 
specifically, they approximate the terminating state distribution as the target distribution. 
However, this is possible \textit{only if} the boundary conditions over terminal nodes are correctly specified. 
When infostate aggregation is excluded from the boundary conditions, then we are mis-specifying the constraints 
that allow a GFlowNet model to approximate the terminating distribution over game outcomes. 

\section{Two New Cases to Cover in Incomplete vs. Complete Information DAGs for IFlowNets}\label{appdx:new-cases}

\begin{figure}[h]
  \centering
  \includegraphics[width=0.35\linewidth]{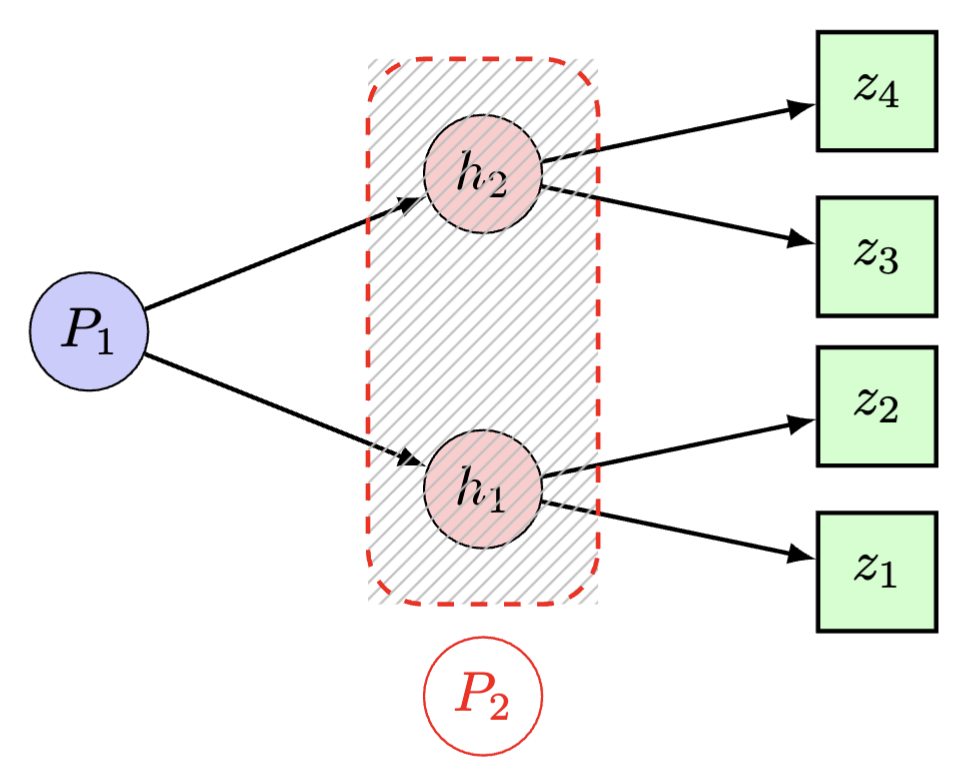}
  \includegraphics[width=0.35\linewidth]{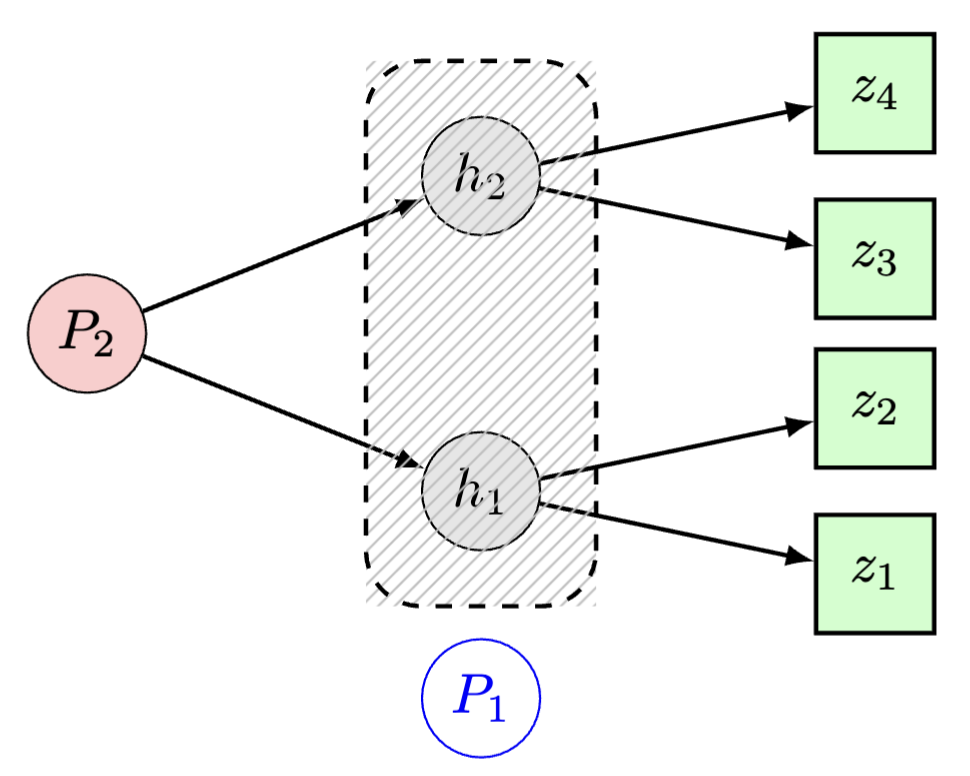}
  \caption{Left: Illustration of why infostate aggregation property necessary. Right: Illustration of higher intra-graph uncertainty in incomplete information games. (See below for details.)}
  \label{fig:case-diagrams}
\end{figure}

\paragraph{Left Scenario of Figure \ref{fig:case-diagrams}:}In the complete information case, $P_2$ could distinguish between $h_1$ and $h_2$. 
Now that $h_1$ and $h_2$ are indistinguishable, they are grayed out; this is because they are in the same infostate, Player $1$ and Player $2$ must take this into account.
The only way this is done in AFlowNet and IFlowNets is through flow constraints over the graph, encoded as the (Generalized) Expected Detailed Balance
constraints.

\paragraph{Right Scenario of Figure \ref{fig:case-diagrams}:} In an incomplete information game, it is possible $P_1$ may find themself at 
an infostate where they cannot distinguish the probability of being in $h_1$ or $h_2$. Therefore, $P_1$ must take this into 
account and pool the total flow mass associated with $h_1$ and $h_2$ in a way that still induces expected reward proportional sampling. 

\section{Review of Expected \& Adversarial Flow Networks}\label{appdx:review}

Adversarial Flow Networks (AFlowNets) are collections of alternating EFlowNets, so we start with EFlowNets. EFlowNets are samplers over a tree $(\G, \State)$, encoded as 
a DAG $\G$ over states $\State$ with an unique initial state $s_0 \in \State$, a set of terminal states $\mathcal{Z}$, 
and a reward function $R$ associated with every $s \in \mathcal{Z}$. Each EFlowNets smooths over transitions it cannot control in the DAG 
by taking expectations, so there is a sampler $i$ for each source of uncertainty. Given sampler $i$, we denote the $i^{\text{th}}$ of $N$ 
EFlowNets by $\mathcal{E}_i := \langle \G, \State, s_0, \mathcal{X}, P_i, R_i \rangle$. We assume for every $\mathcal{E}_i$ that 
$\State$ partitions non-terminal states into \textit{agent states}, denoted by $\State_{i}$, and \textit{environment states}, denoted 
$\State_{env}$ or $\State_{-i}$, where $-i \in [N] \setminus i$, so we may write  
$\State \setminus \mathcal{X} = \State_{\text{agent $i$}} \bigdotcup \hspace{1mm} \State_{\text{Env}} := \State_{i} \bigdotcup \hspace{1mm} \State_{-i}$.
When we are considering the $i^{\text{th}}$ sampler, the other samplers are treated as a product distribution over environment states, i.e., $ P_{Env}(s' \vert s) = P_{-i}(s' \vert s), \forall s' \in Ch(s), \forall s \in \State_{Env}, \forall i \in [N]$. 
 
An \textit{agent policy} is defined to be a collection of distributions over the children of next-possible states, i.e., $P_{agent}(\cdot \vert s)$ 
 for every $s \in \State_{agent}$ (or $P_{i}(\cdot \vert s)$ 
 for every $s \in \State_{i}$). A \textit{forward policy} $P_{F}$ is the forward-sampling distribution, started from $s_0$, determined 
 by alternating all $P_{agent}$ and $P_{env}$. Finally, an AFlowNet is defined with respect to a \textit{collection} of reward functions $R_i$ and a \textit{collection} 
 of all $i$ agents' states $S_i$, $\mathbb{A} := \cup_{i \in [N]} \mathcal{E}_i = \langle \G, s_0, \bigdotcup_{i \in [N]} \hspace{1mm} \State_i, \mathcal{X}, (R_i)_{i \in [N]} \rangle$. 
 In two-player, zero-sum games, AFlowNets are trained by minimizing a log ratio objective constructed from the branch-adjusted trajectory balance (TB) objective
 that is induced by expected detailed balance (EDB) constraints over the graph (summarized by constraints in Eqn.s \ref{eqn:D1}-\ref{eqn:D3}, but with complete information states):

\begin{equation}
    Z_0 \prod_{i: s_i \in \State_{1}} P_{1}(s_{i+1 \vert s_i}) = R_{1}(z)B_{2}(z) \prod_{i: s_i \in \State_{2}} P_{2}(s_{i+1 \vert s_i}),
\end{equation}

where $Z_0$ is a normalizing constant, $R_i$ is the branch-adjusted reward function per-player, and $B_{i}(z) = \prod_{k: s_k \in \State_k} \vert Ch(s_k) \vert$ is agent $i$'s branching factor. 

\section{Quantal Response Equilibria}\label{appdx:qre}

In this section we sketch IFlowNets' connection to quantal response equilibria (QRE). First we note that while 
\cite{jiralerspong_expected_2024} suggest this connection for AFlowNets in complete information 
games, their work only considers QRE in complete information games.

Specifically, we focus on agent QREs (AGREs) \citep{mckelvey_quantal_1998}, which are QREs for extensive-form games like 
those considered in this work (note that QRE for normal form games need not match QREs in extensive-form games). 
AQREs can be viewed as entropy-regularized Nash equilibria, and in incomplete information games, as 
a type of Bayesian equilibrium for a game with perturbed payoffs \citep{sokota_unified_2023}. While we do not empirically or theoretically study AQRE properties of IFlowNets in this work, we 
briefly sketch its relation to AQREs and hope future studies will rigorously analyze these connections. 

Given that IFlowNets reproduce AFlowNet properties for incomplete information games, we know IFlowNet 
agent policies follow $P_i (I_i) \propto \E_{P_{-i}}F_i(I_i)$ (see Lemma \ref{lemma:FM}), so that like \citet{jiralerspong_expected_2024}'s 
AFlowNets, IFlowNets have an \textit{expected flow matching} property over environment dynamics 
and other players strategies, $P_{-i}$. In words, each agent obtained from an IFlowNet will 
sample actions, on average, that are proportional to the learned flow function; following this 
stepwise-sampling as a strategy then implies that IFlowNet players will sample trajectories of play 
that are proportional to expected terminal, branch-adjusted rewards.

However, if all agent policies follow this structure $\forall i$,
then $P_i (I_i) \propto \E_{P_{-i}}F_i(I_i)$ is \textit{some} AQRE \citep{mckelvey_quantal_1998,sokota_unified_2023}. Therefore, if we get to choose the particular branching factor 
$B_i$ and reward $R_i = B_i R^{0}_i$ for the game, then IFlowNets encode the subset of AQREs 
that are expressible by $R_i$. 

For example, suppose an IFlowNet satisfies that $P_i (I_i) \propto \E_{P_{-i}}F_i(I_i)$ and 
$\E_{P_{-i}}F_i(I_i)$ is a logistic function. Then all players are sampling strategies proportional 
to a logistic function of the reward (or utilities), so immediately we know we have estimated some logit AQRE \citep{mckelvey_quantal_1998}. 
The advantage with IFlowNet-based estimation is that we can make this argument for any AQRE 
that is estimable by $\E_{P_{-i}}F_i(I_i)$. Therefore, in principle, one should be able to carefully choose $B_i$ and transformations 
of the reward to find various AQREs. 

\section{Technical Appendices}\label{appdx:technical}


\subsection{Expectation-Based Aggregation Fails to Induce Flow Matching over Infostates}\label{appdx:exp-failure}

\begin{theorem}(\textit{Expectation-Based Aggregation over Infosets via DEDB Invalidates Flow Matching over Infostates})
   
\end{theorem}
\begin{proof}

Suppose $P_{env}$ is known. Under DEDB, fixing some $a \in \A(I)$

\begin{align}
    F(I) P_{agent}(a\vert I) &= \E_{h \sim P_{env}}[F(h,a)] \\
    \iff F(I) & = \sum_{a \in \A(I)} \sum_{h \in I} P_{env}(h) F(h,a)
\end{align}

But for FM to hold, we need the following to be satisfied:

\begin{align}
    \sum_{a \in \A(I)} \sum_{h \in I} P_{env}(h) F(h,a) &\overset{\text{need}}{=} \sum_{a \in \A(I)} \sum_{h \in I} F(h,a),  \hspace{3mm} \forall (h, a, I) \\
    \iff \sum_{a \in \A(I)} \sum_{h \in I} F(h,a) \Big[P_{env}(h) - 1\Big] &= 0,  \hspace{31mm} \forall (h, a, I) \\
    \iff P_{env}(h) &= 1,  \hspace{31mm} \forall (h, I) \\
    \therefore \sum_{h \in I}P_{env}(h) &= \vert I \vert,\hspace{29mm} \forall (h, I)\\
    \Rightarrow\!\Leftarrow
\end{align}

because $P_{env}$ is a probability distribution that should sum to $1$ for any $h$ or $I$. Therefore, re-applying expectation-based aggregation via the DEDB constraints cannot produce the FM property over infosets. 
    
\end{proof}

\subsection{Infoset Flow Matching}\label{appdx:infoset-fm}

\begin{lemma}(\textit{Generalized Expected Detailed Balance Conditions Imply Flow Matching over Infosets.})\label{lemma:FM} 

\end{lemma}
\begin{proof}
    Suppose \ref{eqn:C1} through \ref{eqn:C4} hold. Then, focusing on \ref{eqn:C3},
    \begin{align}
        F(I) P_{agent}(a \vert I) & = \sum_{h' \ni a} F(h') \\
        \iff F(I) &=\sum_{a \in \mathcal{A}(I)} \sum_{h' \ni a} F(h') \\
        & \overset{\text{def.}}{=} \sum_{a \in \mathcal{A}(I)} \sum_{(h,a) \ni a} F(h,a) \\
        & \overset{\text{def.}}{=} \sum_{I' \in Ch(I)}F(I') \\ 
        \therefore F(I) &= \sum_{I' \in Ch(I)}F(I')
    \end{align}

Hence the result. 
\end{proof}

We note that because the GEDB repair and recover the EDB conditions with infoset-based 
flow-matching, one can apply similar arguments as Proposition $5$ in \citet{jiralerspong_expected_2024}. After reparameterizing to 
infostates, we have that $P_i \propto F_i$, for each agent $i$, and then by following the argument for 
Proposition $5$ using the GEDB, we inherit the same expected flow matching property (with the expectation taken 
over other players IFlowNet strategies and environment dynamics).

\subsection{GEDB Constraints Generalize the EDB Constraints}\label{appdx:gedb-edb}

\begin{lemma}(\textit{The GEDB Constraints Recover the EDB Constraints in Complete Information Settings.})\label{lemma:gedb_generalize_edb} 

\end{lemma}
\begin{proof}
   Suppose \ref{eqn:C1} through \ref{eqn:C4} hold. We want to show that \ref{eqn:C1} through \ref{eqn:C4} correspond to and recover the following constraints.

    \begin{align}
        F(z) &= R(z)\label{eqn:E1}\tag{E1}, \hspace{5mm} (\forall z \in \mathcal{Z}) \\ 
        F(s) &= \E_{s' \sim P_{env}(\cdot \vert s)} \Big [ F(s') \Big]\label{eqn:E2}\tag{E2}, \hspace{5mm} ( \forall s, s' \in \mathcal{S}) \\
        F(s) P_{agent}(s' \vert s) & =  F(s')\label{eqn:E3}\tag{E3}, \hspace{5mm} ( \forall s, s' \in \mathcal{S}) 
    \end{align}

\noindent First note that we leave out a condition corresponding to \ref{eqn:C4}, as it is vacuously true when all infosets are singletons: $F(s) = \sum_{s \in s} F(s)$, which is a tautology. Next, \ref{eqn:C1} and \ref{eqn:E1} automatically correspond to one-another, as the flow of any terminal is always taken to be the reward achieved at that terminal node. Now we focus on \ref{eqn:C2} and \ref{eqn:C3}. Suppose $\I = \State$ and consider \ref{eqn:C2}. Then we know $I = s$ for some $s$ and $I' = s'$ for some $s'$.

\begin{align}
    \implies F(I) = F(s) = \E_{I' \sim P_{\cdot \vert s}}[F(I')] = \E_{s' \sim P_{\cdot \vert s}}[F(s')] \\
    \therefore F(s) = \E_{s' \sim P_{\cdot \vert s}}[F(s')],
\end{align}

\noindent recovering \ref{eqn:E2}. Finally suppose $\I = \State$ and consider \ref{eqn:C3}. 

\begin{align}
    F(I)P_{agent}(a \vert I) &= \sum_{h' \ni a} F(h') \\
    \implies F(s)P_{agent}(a \vert s) &= \sum_{s' \ni a} F(s') = F(s')
\end{align}

\noindent The implication follows by definition definition, as each $s$ or $s'$ is a singleton, so the summation on the right-hand side is over a single node. Finally, $s' = (s,a)$ by construction, so 
\begin{align}
    \implies F(s)P_{agent}(a \vert s) &= F(s)P_{agent}(s' \vert s) = F(s') \\ 
    \therefore  F(s)P_{agent}(a \vert s) &= F(s'),
\end{align}
\noindent which proves \ref{eqn:C3} implies \ref{eqn:E3}. Therefore, \ref{eqn:C1} through \ref{eqn:C4} recover and generalize \ref{eqn:E1} through \ref{eqn:E3}, as required. 
\end{proof}

\newpage

\subsection{IFlowNets and Optimality Criteria}

In this section, we step through the optimality criteria and desiderata outlined by \citet{jiralerspong_expected_2024} to show that the GEDB satisfy them. For completeness, we step through the proofs, but the primary differences are in how to handle the re-parameterized infosets and making sure that the original constraints satisfy flow matching (FM). We start by considering the case when there is $n=1$ player vs. the environment, or a Nature player.

Note that once our GEDB are shown to be valid, any sampler satisfying them also satisfies expected reward proportional sampling, under the dynamics 
of the game and other players, immediately due to Proposition $5$ under the GEDB constraints on infostates \citep{jiralerspong_expected_2024}.

\begin{theorem}(\textit{There exists an unique $(F, P_{agent})$ pair satisfying GEDB for $n=1$})\label{thm:unique-single-policy}

\begin{proof}
    In analogy to the EFlowNets case of AFlowNets, consider a single agent navigating an incomplete information game tree $G$ by alternating moves with a Nature player, encoded by $P_{env}$. Then the two distributions characterizing $P_F$, the forward policy, are $P_{agent}$, the agent's policy, and $P_{env}$, the environment dynamics. 

In parallel to \citet{jiralerspong_expected_2024}'s Proposition $1$, proving this relies on two characteristics: a recurrence on the flow function $F$ induced by the GEDB and flow matching being satisfied by the full forward policy. By the hypothesis, $F$ satisfies the GEDB below. 

\begin{equation}
    F(I)=  \begin{cases}

    \sum_{I' \in Ch(I)} F(I'), & I \in \I_{agent} \\

     \E_{I' \sim P_{env}(\cdot \vert I)}\Big[F(I')\Big] & I \in \I_{env}\\
     
     \sum_{h \in I}F(h), & I \in \I_{agent} \\

     R(I) & I \in \mathcal{Z} 

\end{cases}
\end{equation}

where we write constraint \ref{eqn:C3} as a sum in the first line of the piecewise function due to Lemma \ref{lemma:FM}. The proof proceeds similarly to Proposition $1$ in \citet{jiralerspong_expected_2024} as long as it satisfies the key requirement of flow matching, and by Lemma \ref{lemma:FM}, the GEDB induce flow matching on infosets. Consequently, as in Proposition $5$, we may once again consider the longest trajectory over the DAG, and it satisfies a recurrence induced by the GEDB. 

\remark Note that in the DEDB constraints, even if a flow function $F$ satisfies DEDB, the flow function is \textit{not} unique exactly because it cannot satisfy flow matching. Consequently, even if one constructs a policy $P_{agent}^{DEDB}$ from a flow function $F^{DEDB}$, there is no unique $(F^{DEDB}, P_{agent}^{DEDB})$ pair satisfying DEDB. 

Next, because we know $\exists! F$ satisfying GEDB, the uniqueness of $P_{agent}$ follows in an argument similar to Proposition $1$ in \citet{jiralerspong_expected_2024}: by construction, $R(z) > 0$ and $F(I) > 0, \forall I$, and the GEDB define a positivity-preserving recurrence. Consequently, we may construct the agent policy from the unique flow function as $P_{agent} \overset{\Delta}{=} \frac{F(I')}{F(I)}$. Now, because $F$ is unique and satisfies GEDB, we also have that 

\begin{equation}
    F(I) P_{agent}(a \vert I) = F(I) \frac{F(I')}{F(I)}  = F(I')
\end{equation}

Our proof departs at this point, and now, because of infoset consistency and the fact that $h' \overset{\text{def.}}{=}(h,a)$ in incomplete information games, we have 

\begin{equation}
    F(I') = \sum_{(h,a) \ni a} F(h, a) = \sum_{h' \ni a}F(h')
\end{equation}

Therefore, we have that 

\begin{equation}
    F(I) P_{agent}(a \vert I) = F(I') = \sum_{h' \ni a}F(h')
\end{equation}

as desired. Therefore, $\exists!(F, P_{agent})$ satisfying GEDB for $n=1$ (but not the DEDB). 

\end{proof}
\end{theorem}

\begin{theorem}(\textit{For all players $i \in [n]$, there exist unique flow functions and policies $(F_i, P_i)$ that satisfy GEDB with respect to their generalized EFlowNet.})

\end{theorem}

\begin{proof}
    We want to show that $\exists!(F_{i}, P_{i})_{i \in [n]}$, $n>1$, satisfying the GEDB, where each player $P_i$ in the game alternates moves and independently plays their local policy $P_i$ against the unknown joint policy of all other players, $P_{-i}$, encoded as $P_{env}$ in Theorem \ref{thm:unique-single-policy}. Here, we will refer to $P_{-i}$ as $P_{j}$, $j \neq i$. 
    
    Similar to Theorem \ref{thm:unique-single-policy}, we start with the recurrence on the flow function $F_i$ induced by the GEDB.

    \begin{equation}
    F_{i}(I)=  \begin{cases}

    \sum_{I' \in Ch(I)} F_{i}(I'), & I \in \I_{i} \\

     \frac{\sum_{I' \in Ch(I)}F_{i}(I')F_{j}(I')}{\sum_{I' \in Ch(I)}F_{j}(I')} & I \in \I_{j}, j \neq i\\
     
     \sum_{h \in I}F_{i}(h), & I \in \I_{i} \\

     R_{i}(I) & I \in \mathcal{Z} 
\end{cases}
\end{equation}
By Theorem \ref{thm:unique-single-policy}, we know that this recurrence induces an unique $(F_i, P_i)$ pair, but we need to check that it satisfies the $n$-player recurrence here. The first case goes through by Lemma \ref{lemma:FM} and the third and fourth cases go through by construction, so we need to verify the second case (which encodes GEDB constraint \ref{eqn:C2}) holds. However, this is equivalent to the same EDB condition as in \citet{jiralerspong_expected_2024}'s Proposition $2$ over a re-parameterized state space $\I$, 
so the same calculation goes through if and only if flow matching holds. However, because the GEDB satisfies flow matching (see Lemma \ref{lemma:FM}), we have 
\begin{equation}
    \frac{\sum_{I' \in Ch(I)}F_{i}(I')F_{j}(I')}{\sum_{I' \in Ch(I)}F_{j}(I')} = \sum_{I' \in Ch(I)} \frac{F_{i}(I') F_{j}(I')}{\sum_{I'' \in Ch(I)}F_{j}(I'')} = \sum_{I' \in Ch(I)} F_{i}(I) P_{j}(I' \vert I) = \E_{I' \sim P_{j}(\cdot \vert I)}\Big[F(I')\Big]
\end{equation}

As desired, this reproduces the recurrence we had in Theorem \ref{thm:unique-single-policy}.
    
\end{proof}

\newpage
\begin{theorem}(\textit{In $n=2$ player zero-sum IFNs, if agent policies $(P_1,P_2)$ and flows $(F_1,F_2)$ jointly satisfy existence and uniqueness with respect to the GEDB, then $F(I) \overset{\Delta}{=} F_{1}(I)F_{2}(I)$ satisfies flow matching with respect to the joint reward $R(x) \overset{\Delta}{=} R_{1}(x)R_{2}(x)$.})
    
\end{theorem}

\begin{proof}

    Following \citet{jiralerspong_expected_2024}, by construction $F_1(z) = R_1(z)$ and $F_2(z) = R_2(z)$ for all $z \in Z$, so $F(z) \overset{\text{def.}}{=}F_1(z)F_2(z) \overset{\text{\ref{eqn:C1}}}{=}R_1(z)R_2(z)$ for all $z \in Z$.

    We now show that flow matching holds when $I \in \I \setminus Z$. WLOG, suppose $I \in \I_1$. Then we have
    \begin{align}
        F(I) &= F_1(I)F_2(I) = F_1(I) \E_{I' \sim P_{1}(I' \vert I)}F_2(I') = F_1(I) \sum_{I' \in Ch(I)}\frac{F_1(I')}{F_1(I)}F_2(I') = \sum_{I' \in Ch(I)} F_{1}(I')F_2(I') \\ 
        &\overset{\text{def.}}{=} \sum_{I' \in Ch(I)}F(I')
    \end{align}

    So the joint flow function satisfies flow matching if their component flows satisfy the GEDB, as required. 
    
\end{proof}

\newpage

    


    





\newpage
\section{Implementation Details \& Extra Figures}\label{appdx:implementation}

For each environment, we parameterized agent policies as independent MLPs and used the training objective below by resampling trajectories from an online replay 
buffer. For $RPS+$ agents used a $3$-layer MLP with $32$ hiddent units, and for Kuhn and Leduc Poker, agents used a $3$-layer MLP with $512$ hidden units. To find final settings, we ran simulation studies using grid search over ranges of values
described by \citet{keshavarzi_comparative_2025} and picked the model with the best exploitability. For comparisons, we then ran all models for 
10,000 of the model-specific training iterations and recorded exploitability and iterations per second.  We use PyTorch and the default settings 
with the Adam optimizer.
All environments were available through the OpenSpiel Python library (\cite{lanctot_openspiel_2020}). In all environments, we work with 
the trajectory balance objective over infostates, which is valid by the flow matching property of the GEDB constraints \citep{madan_learning_nodate}.

Suppose we have a trajectory starting at an initial state $h_0$ or infostate $I_0$ (or batch of trajectories) $\tau$, such that each infostate alternates as players alternate moves until 
players reach the end of the game $z$. At the end of the game, players receive zero-sum rewards $R^{0}_{1}(z):= \exp(u_{1}(z))$ and $R^{0}_{2}(z):=\exp(u_{2}(z))$, where we take $\exp(\cdot)$
because generative flow network models assume positive rewards, and because we will use a log-transformed objective \citep{bengio_gflownet_nodate}. Following \citet{jiralerspong_expected_2024},
we use a branch-adjustment factor, $B_{i}(z) = \prod_{j: I_j \in I_i} \vert Ch(I_{j}) \vert$, which penalizes equivalent rewards that occur later in the game so that agents prioritize sooner wins. 
The final reward we work with is $R_i(z) := R^{0}_{z} / B_{i}(z)$. Let $Z_{\theta}$ be a learned normalizing constant. The objective we use is
\begin{align*}
  \Ell_{\theta}(\tau) = \log \Bigg(\frac{Z_{\theta} \prod_{I_t \in I_1} P^{\theta}_{1}(I_{t+1}\vert I_t)}{R_{1}(z)B_{2}(z)\prod_{I_{t} \in I_2}P^{\theta}_{2}(I_{t+1}\vert I_t)} \Bigg)
\end{align*}

\subsection{RPS+}

\begin{figure}[htp] 
  \centering
  \subfigure[]{\includegraphics[width=0.4\textwidth]{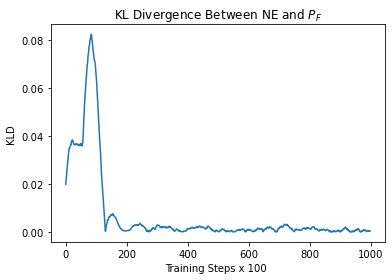}}
  \subfigure[]{\includegraphics[width=0.4\textwidth]{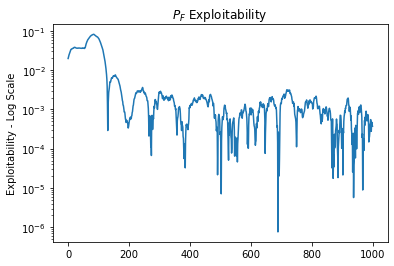}}
  \caption{(a) KLD of IFN policy vs. NE strategy  (b) Exploitability in RPS+}
  \label{fig:weighted-rps-kld}
\end{figure}

\subsection{Kuhn Poker}

\begin{figure}
  \centering
  \includegraphics[width=0.85\linewidth]{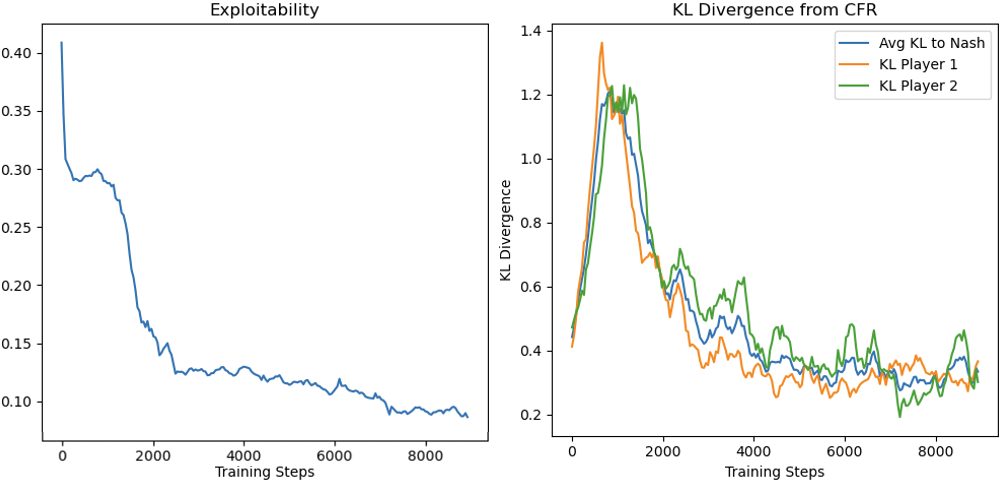}
  \caption{IFN exploitability and KLD from CFR solution. (Exploitability and CFR computed with OpenSpiel utilites in Kuhn OpenSpiel environment.)}
  \label{fig:kuhn}
\end{figure}



\end{document}